\documentclass[twoside,11pt]{article}
\usepackage{jair, theapa, rawfonts}

\jairheading{1}{1993}{1-15}{6/91}{9/91}
\ShortHeadings{DP-NTK for Privacy-Preserving Data Generation}
{Yang, Adamczewski, Sutherland, Li, \& Park}
\firstpageno{25}
\usepackage[textsize=tiny]{todonotes}
\usepackage{amsthm}
\usepackage{bm}
\usepackage{xcolor}
\usepackage{amsmath}
\usepackage{amssymb}
\usepackage{mathtools}
\usepackage{amsthm}
\usepackage{booktabs}
\usepackage{multirow}
\usepackage{standalone}

\usepackage{float}
\usepackage{graphicx}
\usepackage{caption}
\usepackage{subcaption}

\usepackage{algorithm}
\usepackage{algorithmic}
\usepackage{url}

\DeclareMathOperator{\MMD}{MMD}
\DeclareMathOperator{\eNTK}{e-NTK}

\DeclareMathOperator\argmin{arg\,min}

\usepackage{multirow}
\usepackage[capitalize,noabbrev]{cleveref}

\begin{document}
\newcommand{\norm}[1]{\left\lVert#1\right\rVert}
\newcommand{\cH}{\mathcal{H}}

\newcommand{\httpsurl}[1]{\href{https://#1}{\nolinkurl{#1}}}

\newcommand{\arXiv}[1]{arXiv:\href{https://arxiv.org/abs/#1}{#1}}

\setlength\marginparwidth{1.5cm}

\newcommand{\Xiaoxiao}[1]{\textcolor{blue}{[Xiaoxiao: #1]}}
\newcommand{\pluseq}{\mathrel{+}=}

\newcommand{\punt}[1]{}

\newcommand{\inv}{^{-1}}
\newcommand{\trp}{{^\top}} 
\newcommand{\onehlf}{\frac{1}{2}}
\newcommand{\tonehlf}{\tfrac{1}{2}}
\newcommand{\Nrm}{\mathcal{N}}
\newcommand{\Tr}{\mathrm{Tr}}
\newcommand{\diag}{\mathrm{diag}}


\newcommand{\vx}{\mathbf{x}}

\newcommand{\vl}{\mathbf{l}}
\newcommand{\var}{\mathrm{var}}
\newcommand{\Dat}{\mathcal{D}}

\newcommand{\Ev}{\mathcal{E}}
\newcommand{\vxloc}{\mathcal{X}}
\newcommand{\vkx}{\mathbf{k_x}}
\newcommand{\vv}{\mathbf{v}}
\newcommand{\vky}{\mathbf{k_y}}
\newcommand{\vf}{\mathbf{f}}
\newcommand{\vb}{\mathbf{b}}
\newcommand{\valpha}{\mathbf{\ensuremath{\bm{\alpha}}}}
\newcommand{\vphi}{\mathbf{\ensuremath{\bm{\phi}}}}
\newcommand{\bmu}{\mathbf{\ensuremath{\bm{\mu}}}}
\newcommand{\vu}{\mathbf{u}}
\newcommand{\vk}{\mathbf{k}}
\newcommand{\bK}{\mathbf{K}}
\newcommand{\Vm}{\mathbb{V}} 
\newcommand{\vone}{\mathbf{1}} 
\newcommand{\vm}{\mathbf{m}}
\newcommand{\vz}{\mathbf{z}}
\newcommand{\vw}{\mathbf{w}}
\newcommand{\vr}{\mathbf{r}} 
\newcommand{\vs}{\mathbf{s}} 
\newcommand{\ve}{\mathbf{e}} 
\newcommand{\vd}{\mathbf{d}} 
\newcommand{\fmap}{\bm{\lambda}{_{map}}}
\newcommand{\bphi}{\bm{\phi}}
\newcommand{\phimap}{{\hat {\bm{\phi}}}{_{map}}}
\newcommand{\fmapstar}{\mathbf{f^*_{map}}}
\newcommand{\muf}{\ensuremath{\mu_f}}
\newcommand{\vmuf}{\mathbf{\ensuremath{\bm{\mu}_f}}}
\newcommand{\vpi}{\mathbf{\ensuremath{\bm{\pi}}}}
\newcommand{\vmu}{\mathbf{\ensuremath{\bm{\mu}}}}
\newcommand{\vtheta}{\mathbf{\ensuremath{\bm{\theta}}}}
\newcommand{\LL}{\ensuremath{\mathcal{L}}}
\newcommand{\vlam}{\bm{\lambda}}
\newcommand{\mR}{\mathbf{R}}
\newcommand{\mW}{\mathbf{W}}
\newcommand{\mZ}{\mathbf{Z}}
\newcommand{\mC}{\mathbf{C}}
\newcommand{\mI}{\mathbf{I}}

\newcommand{\vX}{\mathcal{X}}
\newcommand{\vK}{\mathbf{K}}
\newcommand{\mJ}{\mathbf{J}}
\newcommand{\vn}{\mathbf{n}}
\newcommand{\R}{\mathbb{R}}

\newcommand{\Em}{\mathbb{E}}

\newcommand{\set}[1]{\{#1\}}
\newcommand{\kml}{{\hat \vk}_{ML}}
\newcommand{\vy}{\mathbf{y}}
\newcommand{\vp}{\mathbf{p}}
\newcommand{\vg}{\mathbf{g}}
\newcommand{\Lprior}{\Lambda_p}
\newcommand{\Lix}{L_{x}}
\newcommand{\thetmap}{\theta_{\mu}}
\newcommand{\Lmap}{\Lambda_\mu}
\newcommand{\thetli}{\theta_{l}}
\newcommand{\Lli}{\Lambda_{l}}
\newcommand{\nsevar}{\sigma^2}
\newcommand{\nsestd}{\sigma}
\newcommand{\vchi}{\bm{\chi}}
\newcommand{\vomega}{\bm{\omega}}
\newcommand{\tr}{^\top}
\newcommand{\vtau}{\mathbf{\ensuremath{\bm{\tau}}}}
\newcommand{\vkappa}{\mathbf{\ensuremath{\bm{\kappa}}}}


\newcommand{\mj}[1]{{\color{blue}{ mijung : #1}}}
\newcommand{\fh}[1]{{\color{green}{ frederik : #1}}}
\newcommand{\ka}[1]{{\color{violet}{ kamil : #1}}}
\newcommand{\marga}[1]{{\color{olive}{ margarita: #1}}}

\newcommand{\vh}{\mathbf{h}}

\theoremstyle{plain}
\newtheorem{theorem}{Theorem}[section]
\newtheorem{proposition}[theorem]{Proposition}
\newtheorem{lemma}[theorem]{Lemma}
\newtheorem{corollary}[theorem]{Corollary}
\theoremstyle{definition}
\newtheorem{definition}[theorem]{Definition}
\newtheorem{assumption}[theorem]{Assumption}
\theoremstyle{remark}
\newtheorem{remark}[theorem]{Remark}

\title{Differentially Private Neural Tangent Kernels (DP-NTK)
\\ for Privacy-Preserving Data Generation}

\author{\name Yilin Yang \email yangyl17@cs.ubc.ca\\
\addr Department of Computer Science, University of British Columbia, Vancouver, Canada.
    \AND
       \name Kamil Adamczewski \email Kamil.m.adamczewski@gmail.com\\
       \addr ETH Zurich, Switzerland.
       \AND
        \name Xiaoxiao Li \email xiaoxiao.li@ece.ubc.ca \\
       \addr Department of Electrical and Computer Engineering, University of British Columbia, Vancouver, Canada.
        \AND
       \name Danica J.\ Sutherland \email 
dsuth@cs.ubc.ca\\
       \addr Department of Computer Science, University of British Columbia, Vancouver, Canada.\\
       Alberta Machine Intelligence Institute, Edmonton, Canada.
       \AND
       \name Mijung Park \email mijungp@cs.ubc.ca \\
       \addr Department of Computer Science, University of British Columbia, Vancouver, Canada,\\
       Alberta Machine Intelligence Institute, Edmonton, Canada. \\
       Department of Applied Mathematics and Computer Science, Technical University of Denmark, Denmark.
       }


\maketitle

\begin{abstract}
Maximum mean discrepancy (MMD) is a particularly useful distance metric for differentially private data generation: when used with finite-dimensional features, it allows us to summarize and privatize the data distribution once, which we can repeatedly use during generator training without further privacy loss.
An important question in this framework is, then, what features are useful to distinguish between real and synthetic data distributions, and whether those enable us to generate quality synthetic data. This work considers using the features of \textit{neural tangent kernels (NTKs)}, more precisely \textit{empirical} NTKs (e-NTKs). We find that, perhaps surprisingly, the expressiveness of the untrained e-NTK features is comparable to that of the features taken from pre-trained perceptual features using public data. As a result, our method improves the privacy-accuracy trade-off compared to other state-of-the-art methods, without relying on any public data, as demonstrated on several tabular and image benchmark datasets.
\end{abstract}

\section{Introduction}
\label{Introduction}
The initial work on differentially private data generation mainly focuses on a theoretical understanding of  the utility of generated data, under strong constraints on the data type and its intended use
\cite{Xiao_DPDR_trough_partitioning,Mohammed_DPDR_for_data_mining,Hardt_simple_practical_DPDR,Zhu_DPDP_survey}. 
However, recent advances in deep generative modelling provide an appealing, more general path towards 
private machine learning:
if trained on synthetic data samples from a sufficiently faithful, differentially private generative model, any learning algorithm becomes automatically private.
One attempt towards this goal is to privatize standard deep generative models by adding appropriate noise to each step of their training process with DP-SGD \cite{DP_SGD}. 
The
majority of the work under this category is based on the generative adversarial networks (GAN) framework. In these approaches, the gradients of the GAN discriminator are privatized with DP-SGD, which limits the size of the discriminator for a good privacy-accuracy trade-off of DP-SGD \cite{DPGAN,DP_CGAN,PATE_GAN,gs-wgan}. 

Another line of work builds off kernel mean embeddings \cite[Chapter 4]{mean-embeddings-survey,berlinet-thomas-agnan}, which summarize a data distribution by taking the mean of its feature embedding into a reproducing kernel Hilbert space.
The distance between these mean embeddings, known as the Maximum Mean Discrepancy or MMD \cite{jmlr-mmd},
gives a distance between distributions;
one way to train a generative model is to directly minimize (an estimate of) this distance \cite<as in>{dziugaite:mmd,gmmn}.

By choosing a different kernel,
the MMD ``looks at'' data distributions in different ways;
for instance,
the MMD with a short-lengthscale Gaussian kernel will focus on finely localized behaviour and mostly ignore long-range interactions,
while the MMD with a linear kernel is based only on the difference between distributions' means.
Choosing an appropriate kernel, then, has a strong influence on the final behaviour of a generative model.
In non-private models, there has been much exploration of adversarial optimization such as GANs \cite[and thousands of follow-ups]{gans},
which corresponds to learning a kernel to use for the MMD \cite{mmd-gan,binkowski:mmd-gans,arbel:smmd}.

In private settings, however, fixing the kernel for a mean embedding has a compelling advantage:
if we pick a fixed kernel with a finite-dimensional feature embedding,
we can privatize the kernel mean embedding of a target dataset \emph{once},
then repeatedly use the privatized embedding during generator training without incurring any further loss of privacy.
This insight was first used in generative modelling in DP-MERF \cite{dp-merf},
which uses random Fourier features \cite{rffs} for Gaussian kernels;
Hermite polynomial features can provide a better trade-off in DP-HP \cite{dp-hp}.
Kernels based on perceptual features extracted from a deep network can also give far better results, if a related but public dataset is available to train that network, via DP-MEPF \cite{dp-mepf}.

When no such public data is available,
we would still like to use better kernels than Gaussians,
which tend to have poor performance on complex datasets like natural images.
In this work, we turn to a powerful class of kernels known as Neural Tangent Kernels \cite<NTKs;>{jacot:ntk,lee:any-depth,arora:exact-comp,chizat-bach:lazy}.
The ``empirical NTK'' (e-NTK) describes the local training behaviour of a particular neural network,
giving a first-order understanding of ``how the network sees data'' 
(see, e.g., \citeR[Proposition 1]{ren:learning-path}, or \citeR{chizat-bach:lazy}).
As the width of the network grows (with appropriate parameterization), this first-order behaviour dominates and the e-NTK converges to a function that depends on the architecture but not the particular weights, often called ``the NTK.''
(\Cref{sec:background} has more details.)
Standard results show that in this infinite limit, training deep networks with stochastic gradient descent corresponds to kernel regression with the NTK.
More relevantly for our purposes,
common variants of GANs with infinitely-wide discriminator networks
in fact reduce exactly to minimizing the MMD based on an NTK \cite{ntk-gans}, providing significant motivation for our choice of kernel.

In addition to their theoretical motivations,
NTKs have proved to be powerful general-purpose kernels for small-data classification tasks \cite{ntk-on-uci},
perhaps because typical neural network architectures provide ``good-enough'' prior biases towards useful types of functions for such problems.
They have similarly been shown to work well for the problem of statistically testing whether the MMD between two distributions is nonzero (and hence the distributions are different) based on samples \cite{ba:ntk-mmd,cheng:ntk-mmd},
where it can be competitive with learning a problem-specific deep kernel \cite{sutherland:opt-mmd,liu:deep-testing}.
The task of generative modelling is intimately related to this latter problem, since we are aiming for a model indistinguishable from the target data \cite<see further discussion by, e.g.,>{arbel:smmd}.

In this work, we use mean embeddings based on e-NTKs
as targets for our private generative model;
we use e-NTKs rather than infinite NTKs
for their finite-dimensional embeddings and computational efficiency.
We show how to use these kernels within a private generative modelling framework,
in (to the best of our knowledge) the first practical usage of NTK methods in privacy settings.
Doing so yields a high-quality private generative model
for several domains,
outperforming models like DP-MERF and DP-HP as well as models based on DP-SGD.










\section{Background} \label{sec:background}

We begin by providing a brief background on differential privacy, MMD, and NTKs.

\subsection{Differential Privacy (DP)}\label{subsec:DP}

A differentially private mechanism guarantees a limit on how much information can be revealed about any single individual's participation in the dataset:
formally,
a mechanism $\mathcal{M}$ is ($\epsilon$, $\delta$)-differentially private 
if 
for all datasets $\Dat$, $\Dat'$ differing by a single entry
and all sets of possible outputs $S$ for the mechanism,
$
\Pr[\mathcal{M}(\Dat) \in S] \leq e^\epsilon \cdot \Pr[\mathcal{M}(\Dat') \in S] + \delta
$.
The level of information leak is quantified by a chosen value of $\epsilon$ and $\delta$;
smaller values correspond to stronger guarantees.

In this work, we use the \textit{Gaussian mechanism} to ensure the output of our algorithm is DP.
Given a function $h(\Dat)$ with outputs in $\R^d$,
we will enforce privacy by adding a level of noise based on the
\textit{global sensitivity} $\Delta_h$ \cite{dwork2006our}.
This is defined as the maximum $L_2$ norm by which $h$ can differ for neighbouring datasets $\Dat$ and $\Dat'$, $\sup \norm{h(\Dat)-h(\Dat')}_2$.
The mechanism's output is denoted by $\widetilde{h}(\Dat) = h(\Dat) + n$, where $n\sim \Nrm(0, \sigma^2 \Delta_h^2\mathbf{I}_d)$.
The perturbed function $\widetilde{h}(\Dat) $ is $(\epsilon, \delta)$-DP, where $\sigma$ is a function of $\epsilon$ and $\delta$.
While $\sigma = \sqrt{2 \log (1.25/\delta)}/\epsilon$ is sufficient for $\epsilon\leq 1$, numerically computing $\sigma$ given $\epsilon$ and $\delta$ using
the \texttt{auto-dp} package of \citeA{wang2019subsampled} gives a smaller $\sigma$ for the same privacy guarantee.

\subsection{Maximum Mean Discrepancy (MMD)} \label{sec:mmd}


For any positive definite function $k : \mathcal X \times \mathcal X \to \R$,
there exists a Hilbert space $\cH$
and a \emph{feature map} $\phi : \mathcal X \to \cH$
such that $k(x, y) = \langle \phi(x), \phi(y) \rangle_\cH$
for all $x, y \in \mathcal X$:
the kernel $k$ is the inner product between feature maps \cite{aronszajn1950theory}.
%
%
%
%
We can also ``lift'' this definition of a feature map to distributions,
defining the \emph{(kernel) mean embedding} of $P$
as
\begin{equation}
    \mu_P = \Em_{x \sim P}[ \phi(x) ] \in \cH
;\end{equation}
it exists as long as $\Em_{x \sim P} \sqrt{k(x, x)} < \infty$ \cite{Smola2007}.
One definition of the MMD is \cite{jmlr-mmd}
\begin{equation}
\MMD(P,Q) = \norm{\mu_P - \mu_Q}_{\cH}
\label{eq:norm2_mmd}
.\end{equation}
If $k$ is
\textit{characteristic} \cite{Sriperumbudur2011}, then $P \mapsto \bmu_P$
is injective,
and so the MMD is a distance metric on probability measures.
If $k$ is not characteristic, it still satisfies all the requirements of a distance except that we might have $\MMD(P, Q) = 0$ for some $P \ne Q$.

Typically, we do not observe distributions $P$ and $Q$ directly,
but rather estimate \eqref{eq:norm2_mmd} based on i.i.d.\ samples
$\{x_1, \dots, x_m\} \sim P$ and $\{x'_1, \dots, x'_n\} \sim Q$.
Although other estimators exist,
it is often convenient to use a simple ``plug-in estimator'' for the mean embeddings:
\begin{align}\label{eq:MMD_est}
\hat\mu_P = \frac1m \sum_{i=1}^m \phi(x_i)
;\;\;
\widehat{\MMD}(P,Q) = \norm{\hat\mu_P - \hat\mu_Q}_\cH
.\end{align}

For many kernels,
an explicit finite-dimensional feature map $\bphi : \mathcal X \to \R^d$ is available;
for a simple example, the linear kernel $k(x, y) = x^\top y$
has the feature map $\bphi(x) = x$
and mean embedding $\bmu_P = \Em_{x \sim P}[x]$.
The MMD then becomes simply the Euclidean distance between mean embeddings in $\R^d$.
This setting is particularly amenable to differential privacy:
we can apply the Gaussian mechanism to the estimate of the relevant mean embedding in $\R^d$.

The commonly-used Gaussian kernel has an infinite-dimensional embedding;
DP-MERF \cite{dp-merf} therefore approximates it with a kernel corresponding to finite-dimensional random Fourier features \cite{rffs},
while DP-HP \cite{dp-hp} uses a different approximation based on truncating Hermite polynomial features \cite{mehler1866ueber}.
DP-MEPF \cite{dp-mepf} chooses a different kernel, one which is defined in the first place from finite-dimensional features extracted from a deep network.
Our e-NTK features are also finite-dimensional, with a particular choice of neural network defining the kernel.

\subsection{Neural Tangent Kernel (NTK)}
The NTK was defined by \citeA{jacot:ntk},
emerging from a line of previous work on optimization of neural networks.
The empirical NTK, which we call e-NTK,
arises from Taylor expanding the predictions of a deep network $f_\theta$, with parameters $\theta \in \R^d$, and is given by
\[
    \eNTK(x, x') = \left[ \nabla_\theta f_\theta(x) \right]^\top \left[ \nabla_\theta f_\theta(x') \right]
.\]
The gradient here is with respect to all the parameters in the network, treated as a vector in $\R^d$.
The $\eNTK$ is a kernel function,
with an explicit embedding $\vphi(x) = \nabla_\theta f_\theta(x)$.

For networks whose output has multiple dimensions (e.g.\ a multi-class classifier),
we use $f_\theta$ above to refer to the \emph{sum} of those outputs;
this ``sum-of-logits'' scheme approximates the matrix-valued e-NTK \cite{mohamadi:pseudo-ntk}, up to a constant scaling factor which does not affect our usage.
In PyTorch \cite{pytorch}, $\nabla_\theta f_\theta(x)$ can be easily computed by simply calling \texttt{autograd.grad} on the multi-output network, which sums over the output dimension if the network's output is multi-dimensional.

As appropriately-initialized networks of a fixed depth become wider,
the e-NTK 
(a) converges from initialization to a fixed kernel, often called ``the NTK,'' independent of $\theta$;
and (b) remains essentially constant over the course of gradient descent optimization
\cite{jacot:ntk,lee:any-depth,arora:exact-comp,chizat-bach:lazy,yang:ntk-init,yang:ntk-dynamics}.
Algorithms are available to compute this limiting function exactly \cite{arora:exact-comp,neuraltangents2020},
but they no longer have an exact finite-dimensional embedding
and can be significantly more computationally expensive to compute than e-NTKs for ``reasonable-width'' networks.

GANs with discriminators in an appropriate infinite limit
theoretically become MMD minimizers using the NTK of the discriminator architecture \cite{ntk-gans}.
In concurrent work, 
\citeA{gantk} train generative models of a similar form directly,
using the infinite NTK;
unfortunately, the particular model they use
-- which is, although they do not quite say so, the infinite-width limit of a least-squares GAN \cite{lsgan} --
appears much more difficult to privatize.
The MMD with infinite NTKs has also proved useful in statistically testing whether the MMD between two datasets is positive, i.e.\ whether the distributions differ
\cite{ba:ntk-mmd,cheng:ntk-mmd}.

Other applications of the e-NTK include predicting the generalization ability of a network architecture or pre-trained network for fine-tuning \cite{wei2022more,bachmann2022generalization,ntk-lm-tuning},
studying the evolution of networks over time \cite{fort2020deep},
and in approximating selection criteria active learning \cite{mohamadi:active-ntk}.
The infinite NTK has seen more widespread adoption in a variety of application areas; the software page of \citeA{neuraltangents2020} maintains an extensive list.


\begin{algorithm}[H]
  \caption{DP-NTK}
  \label{alg:the_alg}
  \begin{algorithmic}
    \STATE \textit{model} $\gets$ Generative model
    \STATE \textit{model\_NTK} $\gets$ NTK model
    \STATE \textit{Dataloader} $\gets$ Dataloader for training set
    \STATE \textit{d} $\gets$ Length of flattened NTK features, i.e. size of NTK model parameters
    \STATE \textit{n} $\gets$ Size of training data obtained by count 
    \STATE $n_c \gets$ Class count
    \STATE \textit{batch\_size} $\gets$ Size of batch
    \STATE \textit{n\_iter}  $\gets$ Number of iterations
    \STATE mean\_emb1 $\gets d\times n_c$ mean embedding of training data (sensitive), initialize at 0
    \FOR {Data, Label in Dataloader}
    \FOR {\textbf{each} x in Data}
        \STATE $\phi(x) \gets$ flattened NTK features for datapoint x calculated using model\_NTK calculated using \texttt{autograd.grad}
        \STATE $\phi(x) \gets \frac{\phi(x)}{||\phi(x)||}$
        \STATE mean\_emb1$[:,\text{Label}[x]] \pluseq \phi(x)$
    \ENDFOR \textbf{ each}
    \ENDFOR
    \STATE noisy\_mean\_emb1 $\gets$ mean embedding of training data with added noise via the Gaussian mechanism
    \FOR {i in n\_iter}
    \STATE \textit{gen\_code, gen\_label} $\gets$ code\_function
    \STATE \textit{gen\_data} $\gets$ model(\textit{gen\_code})
    \STATE $\text{mean\_emb2} \gets d\times n_c$ mean embedding of generated data (non-sensitive)
    \FOR {\textbf{each} x in gen\_data}
        \STATE $\phi(x) \gets$ flattened NTK features for datapoint x calculated using model\_NTK calculated using \texttt{autograd.grad}
        \STATE $\phi(x) \gets \frac{\phi(x)}{||\phi(x)||}$
        \STATE $\text{mean\_emb2}[:,\text{gen\_label}[x]] \pluseq \phi(x)$
        \ENDFOR \textbf{ each}
        \STATE $\text{loss} = ||\text{noisy\_mean\_emb1} - \text{mean\_emb2}||_{2,1}^2$
        \STATE backward using loss then step
        \STATE per-iter evaluation
    \ENDFOR
    \STATE final accuracy evaluation
  \end{algorithmic}
\end{algorithm}

\section{The DP-NTK model}
Here we present our method, which we call \textit{differentially private neural tangent kernels (DP-NTK)} for privacy-preserving data generation. 


We describe our method in the class-conditional setting; it is straightforward to translate to unconditional generation by, e.g., assigning all data points to the same class.

Following \citeA{dp-merf,dp-hp,dp-mepf}, we encode class information in the MMD objective by defining a kernel for the joint distribution over the input and label pairs in the following way.
Specifically, we define 
\[ k((\vx,\vy), (\vx', \vy')) = k_\vx(\vx, \vx') k_\vy(\vy, \vy')
,\]
where $k_\vy(\vy, \vy') = \vy^\top \vy'$ for one-hot encoded labels $\vy, \vy'$.
For the kernel on the inputs, we use the ``normalized'' e-NTK of a network $f_\theta : \mathcal X \to \R$:
\begin{gather*}
k_\vx(\vx, \vx') = \vphi(\vx) \trp \vphi(\vx),
\\
\mbox{where }
\vphi(\vx) = \nabla_\theta f_\theta(\vx) / \norm{\nabla_\theta f_\theta(\vx)}
.\end{gather*}
The normalization is necessary to bound the sensitivity of the mean embedding.
Recall that if a network outputs vectors, e.g.\ logits for a multiclass classificaton problem, we use $f_\theta$ to denote the sum of those outputs \cite{mohamadi:pseudo-ntk}.

Given a labelled dataset $\{(\vx_i, \vy_i)\}^m_{i=1}$, we represent the mean embedding of the data distribution as
\begin{align}
\label{eq:me_data}
\hat\vmu_{P} = \frac{1}{m}\sum^m_{i=1} \vphi(\vx_i)\vy_i^\top
;\end{align}
this is a $d \times c$ matrix, where $d$ is the dimensionality of the NTK features and $c$ the number of classes.\footnote{This corresponds to a feature space of $d \times c$ matrices, where we use the Frobenius norm and inner product; though this is itself a Hilbert space, this exactly agrees with ``flattening'' these matrices into $\R^{dc}$ and then operating in that Euclidean space.}

\begin{proposition}
    The global sensitivity of the mean embedding \eqref{eq:me_data} is $\Delta_{\bmu_P} = 2 / m$.
\end{proposition}
\begin{proof}
Using the definition of the global sensitivity,
which bounds for $\Dat, \Dat'$ which differ in only one entry,
and the fact that $y$ is a one-hot vector
and $\vphi$ is normalized,
we have that
\begin{align*}
    \Delta_{\hat{\mu}_p}
    &= \sup_{\Dat, \Dat'} \norm{ \frac{1}{m}\sum^m_{i=1} \vphi(\vx_i) \vy_i^\top - \frac{1}{m}\sum^m_{j=1}\vphi(\vx_j') \vy_j'^\top }_F \\
    &= \sup_{(\vx, \vy ),( \vx', \vy' )}
    \norm{ \frac{1}{m}\vphi(\vx) \vy^\top - \frac{1}{m}\vphi(\vx') {\vy'}^\top }_F
    \\
    &\leq \frac{2}{m} \sup_{(\vx, \vy )} \norm{\vphi(\vx) \vy^\top }_F
    \\
    &= \frac{2}{m} \sup_{\vx} \norm{\vphi(\vx)} 
    \\
    &= \frac{2}{m}
.\qedhere\end{align*}
\end{proof}

We use the Gaussian mechanism to privatize the mean embedding of the data distribution,
\begin{align*}
\Tilde{\bmu}_{P}
&= \hat{\bmu}_{P} + \mathcal{N}(0, \sigma^2\Delta_{\bmu_P}^2\mathbf{I}_d)
\\&= \hat{\bmu}_{P} + \mathcal{N}\left(0, \frac{4 \sigma^2}{m^2} \mathbf{I}_d \right)
.\end{align*}
The privacy parameter $\sigma$ is a function of $\epsilon$ and $\delta$, which we numerically compute using the \texttt{auto-dp} package of \citeA{wang2019subsampled}.



Our generator $G$ produces an input conditioned on a generated label $\vy_i'$, i.e.\ $G(\vz_i, \vy_i') \mapsto \vx_i'$, where $\vy_i'$ is drawn from a uniform distribution over the $c$ classes and each entry of $\vz_i \in \mathbb{R}^{d_z}$ is drawn from a standard Gaussian distribution.   
%
Similar to \eqref{eq:me_data}, the mean embedding of the synthetic data distribution is given by 
\begin{align}
\hat{\boldsymbol{\mu}}_{Q} = \frac{1}{n}\sum^n_{i=1} \phi(\vx'_i) \, {\vy'_i}^\top
.\end{align}
%
%
%
Our privatized MMD loss is given by
\begin{equation} \label{eq:priv-mmd}
\widetilde{\text{MMD}}^2(P, Q) = \norm{\Tilde{\boldsymbol{\mu}}_{P} - \hat{\boldsymbol{\mu}}_{Q}}^2_F, 
\end{equation}
where F is the Frobenius norm.
This loss is differentiable (as was proved in the NTK case by \citeR{ntk-gans}; we compute the derivative with standard automatic differentiation systems).
We minimize \eqref{eq:priv-mmd} with respect to the parameters of the generator $G$ with standard variations on stochastic gradient descent,
where $\Tilde{\boldsymbol\mu}_P$ remains a constant vector
but $\hat{\boldsymbol\mu}_Q$ is computed based on a new batch of generator outputs at each step.
%

For class-imbalanced and heterogeneous datasets, such as the tabular datasets we consider, we use the modified mean embeddings of \citeA[Sections 4.2-4.3]{dp-merf}, replacing their random Fourier features with our e-NTK features. 


\section{Theoretical analysis}

Existing results on this class of models measure the quality of a generative model in terms of its squared MMD to the target distribution.
How useful is that as a metric?
The origin of the name ``maximum mean discrepancy'' is because we have in general that
\[
    \MMD(P, Q) = \sup_{f : \norm f_k \le 1} \Em_{x \sim P} f(x) - \Em_{y \sim Q} f(y)
,\]
where $\norm f_k$ gives the norm of a function $f : \mathcal X \to \R$ under the kernel $k$.
If the squared MMD is small,
no function with small norm under the kernel $k$
can strongly distinguish $P$ and $Q$.
But when our kernel is the NTK of a wide neural network,
the set of functions with small kernel norm
is exactly the functions which can be learned by gradient descent in bounded time
(also see more detailed discussion by \citeR{cheng:ntk-mmd}).
Thus, if $P$ is similar to $Q$,
(roughly) no neural network of the architecture used by our discriminator can distinguish the two distributions.
Although our e-NTK kernel is finite-dimensional and hence not characteristic,
this implies that distributions with small MMD under the NTK should model the distribution well for practical usages.

We can bound this MMD between the target distribution and the private generative model in two steps:
the gap between the private and the non-private model,
plus the gap between the non-private model and the truth.

For the extra loss induced by privacy, we can use
the theoretical analysis of \citeA{dp-mepf},
which is general enough to also apply to DP-NTK.
Their main result shows the following in our setting.

\begin{proposition}
    Fix a target distribution $P$,
    and let $\mathcal Q$ denote some class of probability distributions,
    e.g.\ all distributions which can be realized by setting the parameters of a generator network.
    
    Let $\tilde Q \in \argmin_{Q \in \mathcal Q} \widetilde{\MMD}(P, Q)$ minimize the private loss \eqref{eq:priv-mmd},
    and let $\hat R \in \argmin_{Q \in \mathcal Q}\\ \widehat{\MMD}(P, Q)$ minimize the non-private loss \eqref{eq:MMD_est}.

    Then the sub-optimality of $\tilde Q$ relative to $\hat R$ is given by\footnote{See their Proposition A.5 for specific constants, and plug in $\operatorname{Tr}(\Sigma) = 4 d \sigma^2 / m^2$, $\norm{\Sigma}_F = 4 \sqrt{d} \sigma^2 / m^2$, $\norm{\Sigma}_\mathit{op} = 4 \sigma^2 / m^2$.}
    \begin{align}
      \widehat\MMD^2(P, \tilde Q) - \widehat\MMD^2(P, \hat R) 
      = \mathcal O_p\left(
        \frac{\sigma^2 d}{m^2}
        + \frac{\sigma \sqrt d}{m} \widehat\MMD(P, \hat R)
      \right)
    .\end{align}
\end{proposition}
The notation $A = \mathcal O_p(B)$ means roughly that with any constant probability, $A = \mathcal O(B)$.
Note that the MMD in the latter term is bounded by $\sqrt 2$,
so we can always achieve a $\mathcal O_p(1 / m)$ rate for the squared MMD.
If $\MMD(P, \hat R) = 0$, though
-- the ``interpolating'' case --
this ``optimistic'' rate shows that the private model is only $\mathcal O_p(1 / m^2)$ worse than the non-private one,
as measured on the particular sample.

Thus, we know that the private minimizer will approximately minimize the non-private loss.
The sub-optimality in that minimization decays at a rate much faster than 
the known rates for the convergence of the non-private minimizer to the best possible model.
The rates obtained by e.g.\ \citeA{dziugaite:mmd} have complex dependence on the generator architecture,
but never give a rate for $\MMD^2(P, \hat R) - \inf_{Q \in \mathcal Q} \MMD^2(P, Q)$ faster than
$\mathcal O_p(1 / \sqrt m)$;
this is far slower than our $\mathcal O_p(1 / m)$ or even $\mathcal O_p(1 / m^2)$ rate for the extra loss due to privacy,
so the private model is (asymptotically) not meaningfully worse than the non-private one.

\section{Experiments}
We test our method, DP-NTK, on popular benchmark image datasets such as MNIST, FashionMNIST, CelebA, and CIFAR10, as well as on 8 benchmark tabular datasets.
For MNIST/FMNIST we use a Conditional CNN as the generator. 
We use ResNet18 as our generator for Cifar10 and CelebA, and a fully connected network for tabular data: for homogenous data, we use 2 hidden layers + ReLU+ batch norm, or 3 hidden layers + ReLU + batch norm plus an additional sigmoid layer for the categorical features for heterogeneous data.

Our code is at 
\url{https://github.com/FreddieNeverLeft/DP-NTK}. 
The readme file of this repository and the section \cref{tab:best_hyper} contain hyper-parameter settings (e.g.,  the architectural choices used for the model whose e-NTK we take) for reproducibility.




\subsection{Generating MNIST and FashionMNIST images}

For our MNIST and FashionMNIST experiments, we choose the e-NTK of a fully-connected 1-hidden-layer neural network architecture with a width parameter $w$, for simplicity. As the NTK features considered here are uniquely related to the width of the NTK network, for MNIST and F-MNIST data we conduct two groups of experiments, one where we test the accuracy of the model under the same Gaussian mechanism noise but with different NTK widths $w$, and the other with the same width ($w = 800$) but with varying privacy noise levels. In this setting, we will also compare results with other DP image-generating methods that do not rely on public data for pre-training, as well as DP-MEPF which uses public data to train a feature extractor model.

Following prior work such as \citeA{dp-merf,dp-hp,dp-mepf}, we quantitatively measure the quality of samples by training a classifier on the synthetic data and then measuring its accuracy when applied to real data. Here, we use two image classifiers: a multilayer perceptron (MLP) and logistic regression.

\begin{table*}[tp]
  \vspace{1ex}
  \centering
  \begin{tabular}{c c c  c c} 
\multicolumn{5}{c}{\textbf{MNIST}}\\
  &   $w = 100$  &  $w = 400$ & $w = 800$ & $w = 1000$ \\
  \hline
  Log Reg &  $0.8423$   & $0.8404$ & $0.84$ & $0.8387$ \\
  \hline
  MLP & $0.8824$  & $0.8811$ & $0.88$ & $0.8805$ \\ 
 \hline
\end{tabular}
\begin{tabular}{c c c c c} 
 \multicolumn{5}{c}{\textbf{FashionMNIST}}\\
  &   $w = 100$   & $w = 400$ & $w = 800$ & $w = 1000$ \\
  \hline
  Log Reg & $0.7640$  & $0.7642$ & $0.7663$ & $0.7643$ \\
  \hline
  MLP & $0.7773$  & $0.7806$ & $0.7838$& $0.7768$  \\ 
 \hline
\end{tabular}
 \caption{Accuracy under different widths, with $(10, 10^{-5})$ DP}
 \label{tab:mnist-widths}
\end{table*}

\begin{figure*}[tp]
    \centering
    \includegraphics[width=0.65\textwidth,height=0.4\textwidth]{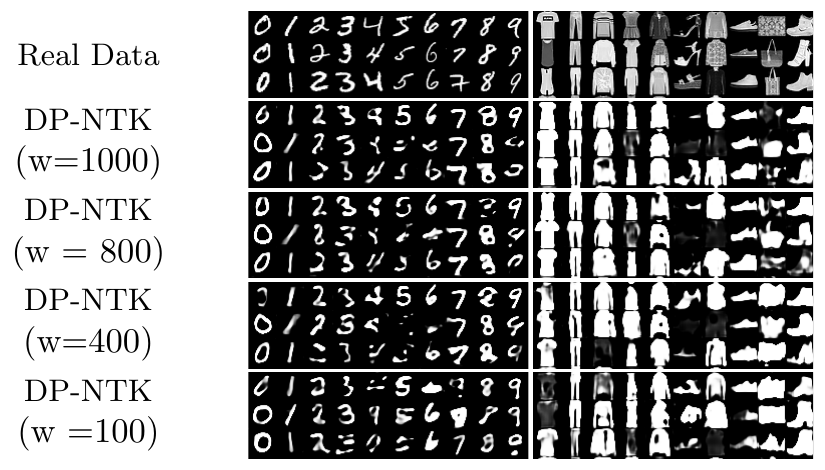}
    \caption{Generated samples of MNIST and FashionMNIST from DP-NTK with different widths $w$; all samples use the same DP noise level ($\epsilon=10$, $\delta = 10^{-5}$).}
    \label{fig:diff_width}
\end{figure*}

\paragraph{Privacy-Width Trade-off}
As discussed e.g.\ by \citeA{dp-hp}, applying the Gaussian mechanism in feature spaces of higher dimension causes our model to become less accurate. This necessitates a trade-off in the feature dimensionality:
large dimensions can lead to overwhelming amounts of added noise,
while small dimensions may be inadequate to serve as a loss for image generation.

\begin{figure*}[t!]
    \centering
    \begin{subfigure}[t]{0.5\textwidth}
        \centering
        \includegraphics[width=\textwidth,height=0.8\textwidth]{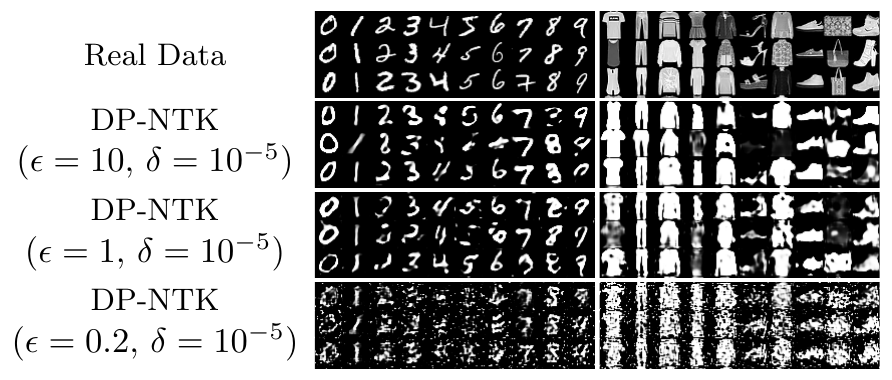}
        \caption{Generated samples for MNIST and FashionMNIST from DP-NTK, with the same width ($w=800$) and different DP levels.}
        \label{fig:same_width}
    \end{subfigure}%
    ~ 
    \begin{subfigure}[t]{0.5\textwidth}
        \centering
        \includegraphics[width=\textwidth,height=0.8\textwidth]{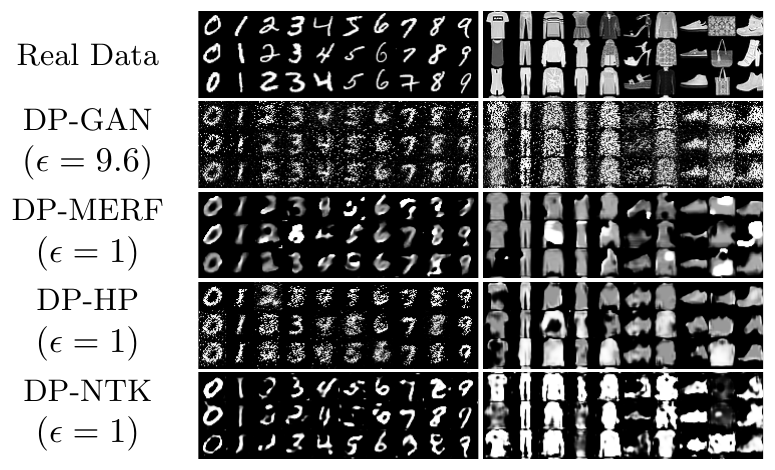}
        \caption{Generated samples for MNIST and FashionMNIST from DP-NTK and comparison models}
        \label{fig:generated_samples}
    \end{subfigure}
    \caption{DP-NTK under different DP levels (left) and comparison results with different models (right) for MNIST and FashionMNIST}
\end{figure*}

\begin{table*}[tp]
\vspace{1ex}
\centering
\scalebox{0.85}{
\begin{tabular}{l|l|ll|ll}
\toprule
\multirow{2}{*}{Method} & \multirow{2}{*}{DP-$\epsilon$} & \multicolumn{2}{l|}{MNIST} & \multicolumn{2}{l}{Fashion-MNIST} \\ \cline{3-6} 
              &           &Log Reg& MLP & Log Reg & MLP \\ \hline
\textbf{DP-NTK (ours)}       & 0.2  & \textit{0.804}& \textit{0.794}& 0.634 & 0.694 \\
DPDM      \cite{dockhorn2023dpdm} & 0.2       &   \textbf{0.81}& \textbf{0.817}& 0.704& \textbf{0.713}\\
PEARL \cite{PEARL}        & 0.2       & 0.762& 0.771& 0.700& 0.708\\
DP-MERF \cite{dp-merf}       & 0.2       &  0.772& 0.768& \textbf{0.714} & 0.696\\\hline
\textbf{DP-NTK (ours)} & 1 & \textit{0.8324}& \textit{0.8620}&\textit{0.7596} & \textit{0.7645}\\
DPDM      \cite{dockhorn2023dpdm} & 1       &   \textbf{0.867}& \textbf{0.916}& \textbf{0.763}& \textbf{0.769}\\
PEARL      \cite{PEARL}   & 1         & 0.76& 0.796& 0.744& 0.740\\
DP-HP      \cite{dp-hp} & 1       &   0.807& 0.801& 0.7387& 0.7215\\
DP-MERF     \cite{dp-merf}  & 1       & 0.769& 0.807& 0.728&  0.738  \\\hline
\textbf{DP-NTK (ours)}        & 10        & \textit{0.8400}& \textit{0.8800}& \textit{0.7663} & \textit{0.7838} \\
DPDM \cite{dockhorn2023dpdm}  & 10        & \textbf{0.908}& \textbf{0.948}&  \textbf{0.811}& \textbf{0.830}\\
PEARL     \cite{PEARL}    & 10        & 0.765& 0.783&  0.726& 0.732\\
DP-Sinkhorn \cite{dp_sinkhorn}  & 10        & 0.828& 0.827&  0.751& 0.746\\
DP-CGAN    \cite{DP_CGAN}   & 10        &   0.6&   0.6& 0.51&  0.5\\
DP-HP      \cite{dp-hp} & 10       &   0.8082& 0.8038& 0.7387& 0.7168\\
DP-MERF     \cite{dp-merf}  & 10        & 0.794& 0.783& 0.755&   0.745  \\
GS-WGAN     \cite{gs-wgan}  & 10        &  0.79&  0.79& 0.68&   0.65 \\
\hline
\end{tabular}}
\caption{Performance comparison on MNIST and F-MNIST dataset averaged over five independent runs, with NTK width fixed at 800. Here $\delta$ is fixed at $10^{-5}$.}
\label{tab:mnist-dps}
\end{table*}

\Cref{fig:diff_width,tab:mnist-widths} show that for both MNIST and FashionMNIST, changing the width does somewhat affect the final accuracy and image quality, but this effect is very minimal. Therefore, for subsequent experiments, we choose a width of $800$ as a good compromise.

\paragraph{Varying privacy levels}
\cref{tab:mnist-dps,fig:same_width} show our model's performance under different levels of privacy. Other than for FashionMNIST with $\epsilon = 0.2$, the performance of our model does not degrade significantly as the privacy requirement becomes more stringent.

\Cref{tab:mnist-dps} also compares with several existing models,
while \cref{fig:generated_samples} visually compares results 
with DP-GAN \cite{DPGAN}, DP-MERF \cite{dp-merf} and DP-HP \cite{dp-hp}.
We see that even with simple architectures, DP-NTK broadly performs better than other high-accuracy models, and generates comprehensible images. Although there is still distance from the current SOTA \cite{dockhorn2023dpdm}, we show that DP-NTK still offers competitive performance under different privacy levels on MNIST and FashionMNIST.
We also compare the performance of DP-MEPF (that relies on public data for pre-training a feature extractor model) and DP-NTK on MNIST and FashionMNIST datasets in \Cref{tab:dpmepfvsdpntk}. 

\begin{table}[t]
\vspace{1ex}
\centering
\scalebox{1}{
\begin{tabular}{ll|ll|ll}
\toprule
\multirow{2}{*}{Method} & \multirow{2}{*}{DP-$\epsilon$} & \multicolumn{2}{l|}{MNIST} & \multicolumn{2}{l}{Fashion-MNIST} \\ \cline{3-6} 
              &           &Log Reg& MLP & Log Reg & MLP \\ \hline
\textbf{DP-NTK (ours)}       & 0.2  & \textbf{0.804}& \textbf{0.794}& 0.634 & 0.694 \\
DP-MEPF \cite{dp-mepf}       & 0.2       &  0.80& 0.76& \textbf{0.73} & \textbf{0.7}\\\hline
\textbf{DP-NTK (ours)} & 1 & \textbf{0.8324}& 0.8620&\textbf{0.7596} & \textbf{0.7645}\\
DP-MEPF     \cite{dp-mepf}  & 1       & 0.82& \textbf{0.89}& 0.75&  0.75  \\\hline
\textbf{DP-NTK (ours)}        & 10        & \textbf{0.8400}& 0.8800& \textbf{0.7663} & \textbf{0.7838} \\
DP-MEPF     \cite{dp-mepf}  & 10        & 0.83& \textbf{0.89}& 0.76&   0.76 \\ \hline 
\end{tabular}}
\caption{Performance comparison on MNIST and F-MNIST dataset for our method with NTK width fixed at 800 and DP-MEPF \cite{dp-mepf}. Here $\delta$ is fixed at $10^{-5}$.}
\label{tab:dpmepfvsdpntk}
\end{table}



\begin{figure}[H]
\centering
\includegraphics[width=0.6\textwidth]{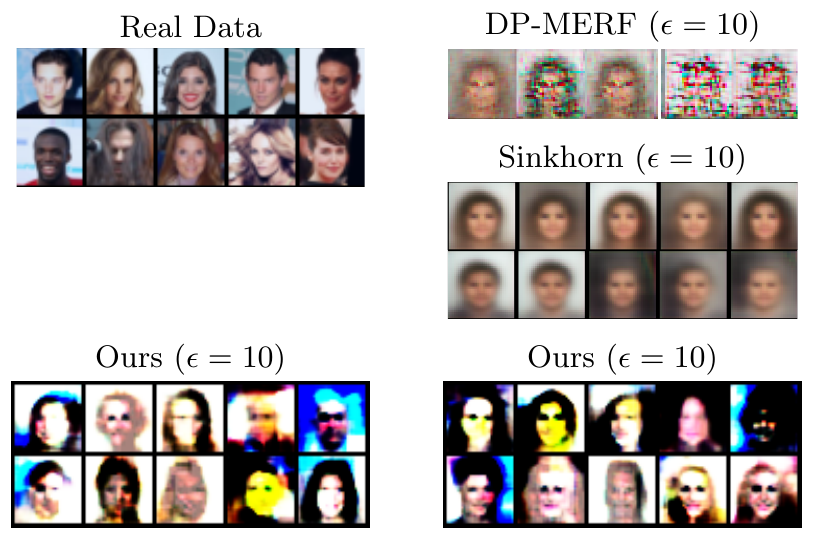}
\caption{Synthetic $32 \times 32$ CelebA samples generated at different levels of privacy. Samples for DP-MERF and DP-Sinkhorn are taken from \citeA{dp_sinkhorn}. Our method yields samples of higher visual quality than the comparison methods. The FID for the proposed method is 75. FID for DP-Sinkhorn is 189. FID for DP-MERF is 274.}
\label{fig:celeba_samples}
\end{figure}

\begin{figure}[H]
\centering
\includegraphics[width=0.5\textwidth]{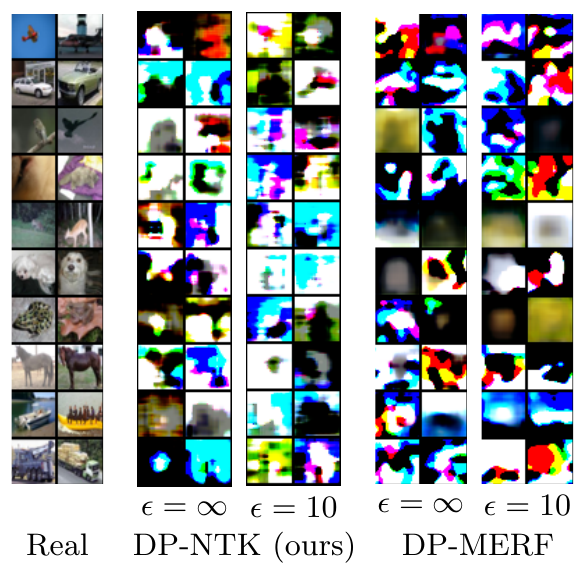}
\caption{The generated and real images for the CIFAR-10 dataset. The FID scores for the proposed method are 104  ($\epsilon=\infty$) and 107  ($\epsilon=10$), respectively. For DP-MERF, they are 127 ($\epsilon=\infty$) and 141 ($\epsilon=10$).}
\label{fig:cifar_samples}
\end{figure}

\subsection{Generating CelebA and CIFAR10 images}

Image datasets beyond MNIST are notoriously challenging for differentially-private data generation. In the Appendix,
\Cref{fig:celeba_samples} on CelebA (at $32 \times 32$ resolution).
On CelebA, DP-MERF and DP-Sinkhorn give blurry, near-uniform samples,
with DP-MERF also containing significant amounts of very obvious pixel-level noise.
Our method, by contrast, generates sharper, more diverse images with some distinguishable facial features,
although generation quality remains far behind that of non-private models. 

\Cref{fig:cifar_samples} shows results on CIFAR-10,
where the difficulty lies in having relatively few samples for a wide range of image objects.
Our samples, while again far from the quality of non-private generators,
are less blurry than DP-MERF and contain shapes reminiscent of real-world objects.

\begin{table*}[t]
\vspace{1ex}
\centering
\scalebox{0.85}{

\begin{tabular}{l *{5}{|cc|cc} }
\toprule
& \multicolumn{2}{ c| }{Real}  & \multicolumn{2}{ c| }{DP-CGAN}  &
\multicolumn{2}{ c| }{DP-GAN}  &
\multicolumn{2}{ c| }{{DP-MERF}} &
\multicolumn{2}{ c| }{DP-HP} & \multicolumn{2}{ c| }{\textbf{DP-NTK}}   \\ 

&\multicolumn{2}{ c| }{} &   \multicolumn{2}{ c| }{($1,10^{-5}$)-DP} & \multicolumn{2}{ c| }{($1,10^{-5}$)-DP} & 
\multicolumn{2}{ c| }{($1,10^{-5}$)-DP} &
\multicolumn{2}{ c| }{($1,10^{-5}$)-DP} &
\multicolumn{2}{ c|}{($1,10^{-5}$)-DP}
\\ 

\hline
\textbf{adult} & 0.786 & 0.683 &  0.509 & 0.444 & 0.511 & 0.445 & 0.642 & 0.524 & 0.688 & \textbf{0.632} & \textbf{0.695} & 0.557\\
\textbf{census} & 0.776 & 0.433   & 0.655 & 0.216 & 0.529 & 0.166 & 0.685 & 0.236 & 0.699 & 0.328 & \textbf{0.71} &	0.\textbf{424}\\
\textbf{cervical} & 0.959 & 0.858  & 0.519 & 0.200 & 0.485 & 0.183 & 0.531 & 0.176 & 0.616 & 0.31 & \textbf{0.631} &	\textbf{0.335}\\
\textbf{credit} & 0.924 & 0.864 &  0.664 & 0.356 & 0.435 & 0.150 & 0.751 & 0.622 & 0.786 & 0.744 & \textbf{0.821} &	\textbf{0.759}\\
\textbf{epileptic} & 0.808 & 0.636 &  0.578 & 0.241 & 0.505 & 0.196 & 0.605 & 0.316 & 0.609 & \textbf{0.554} & \textbf{0.648} &	0.326\\
\textbf{isolet} & 0.895 & 0.741 & 0.511 & 0.198 & 0.540 & 0.205 & 0.557 & 0.228 & \textbf{0.572} & \textbf{0.498} & 0.53	& 0.205  \\ 

& \multicolumn{2}{ c| }{F1}
& \multicolumn{2}{ c| }{F1}
& \multicolumn{2}{ c| }{F1}
& \multicolumn{2}{ c| }{F1}
& \multicolumn{2}{ c| }{F1}
& \multicolumn{2}{ c|  }{F1}

\\ 
\textbf{covtype} & \multicolumn{2}{ c| }{0.820} &
\multicolumn{2}{ c| }{0.285} &
\multicolumn{2}{ c| }{0.492} &
\multicolumn{2}{ c| }{0.467} &  \multicolumn{2}{ c| }{0.537} &  \multicolumn{2}{ c| }{\textbf{0.552}}\\

\textbf{intrusion} & \multicolumn{2}{ c| }{0.971} & 
\multicolumn{2}{ c| }{0.302} & 
\multicolumn{2}{ c| }{0.251} &
\multicolumn{2}{ c| }{\textbf{0.892}} & \multicolumn{2}{ c| }{0.890}  & \multicolumn{2}{ c| }{0.717} \\ 
\bottomrule
\end{tabular}}
\caption{Performance comparison on Tabular datasets averaged over five independent runs. The top six datasets contain binary labels while the bottom two datasets contain multi-class labels. The metric for the binary datasets is ROC and PRC and for the multiclass datasets F1 score.
}
\label{tab:summary_tabular}
\end{table*}
\subsection{Generating tabular data}
For DP-NTK, we explore the performance of the algorithm on eight different imbalanced tabular datasets with both numerical and categorical input features. The numerical features on those tabular datasets can be either discrete (e.g.\ age in years) or continuous (e.g.\ height), and the categorical ones may be binary (e.g.\ drug vs placebo group) or multi-class (e.g.\ nationality).

For evaluation, we use ROC (area under the receiver characteristic curve) and PRC (area under the precision recall curve) for datasets with binary labels, and F1 score for datasets with
multi-class labels. \Cref{tab:summary_tabular} shows the average over the 12 classifiers trained on the generated samples (also averaged over 5 independent seeds). For most of the datasets, DP-NTK outperforms the benchmark methods. 

\section{Summary and Discussion}
We introduced DP-NTK,
a new model for differentially private generative modelling.
NTK features are an excellent fit for this problem,
providing a data-independent representation that performs well on a variety of tasks:
DP-NTK greatly outperforms existing private generative models 
(other than those relying on the existence of similar public data for pre-training)
for relatively simple image datasets and a majority of tabular datasets that we evaluated.

There are, however, certainly avenues that could improve the DP-NTK model in future work.
One path might be finding a way to use the infinite NTK, which tends to have better performance on practical problems than e-NTKs at initialization \cite{ntk-on-uci}.

When related public datasets are available for pretraining,
it would also be helpful for DP-NTK to be able to incorporate that information,
as \citeA{dp-mepf} have argued is important.
Indeed, by exploiting this information, their model outperforms ours on some image datasets. 
One strategy would be to take the e-NTK of a pretrained model on that related dataset;
our initial attempt at this did not help much, however,
and we leave further investigation to future work.


\vskip 0.2in
\bibliography{DP-NTK}
\bibliographystyle{theapa}


\appendix



\section{Hyperparameters Used in Experiments}
\begin{table}[H]
\label{tab:best_hyper}
\vspace{1ex}
\centering
\scalebox{1}{
    \begin{tabular}{|l|l|l|l|l|l|l|l|}
    \hline
        dataset & iter & d\_code & ntk width & batch & lr & eps & architecture \\ \hline
        dmnist & 2000 & 5 & 800 & 5000 & 0.01 & 10, 1, 0.2 & fc\_1l \\ \hline
        fmnist & 2000 & 5 & 800 & 5000 & 0.01 & 10, 1, 0.2 & fc\_1l \\ \hline
        celeba & 20000 & 141 & 3000\_200 & 1000 & 0.01 & 10 & fc\_2l \\ \hline
        cifar10 & 40000 & 201 & 3000\_200 & 1000 & 0.01 & 10 1 & fc\_2l \\ \hline
        cifar10 & 20000 & 31 & 800\_1000 & 200 & 0.01 & None & fc\_2l \\ \hline
        adult & 50 & 11 & 30\_200 & 200 & 0.01 & 1 & cnn\_2l \\ \hline
        census & 2000 & 21 & 30\_20 & 200 & 0.01 & 1 & cnn\_2l \\ \hline
        cervical & 500 & 11 & 800\_1000 & 200 & 0.01 & 1 & cnn\_2l \\ \hline
        credit & 500 & 11 & 1500 & 200 & 0.01 & 1 & fc\_1l \\ \hline
        epileptic & 2000 & 101 & 50\_20 & 200 & 0.01 & 1 & cnn\_2l \\ \hline
        isolet & 1000 & 21 & 10\_20 & 200 & 0.01 & 1 & cnn\_2l \\ \hline
        covtype & 1000 & 101 & 100\_20 & 200 & 0.01 & 1 & cnn\_2l \\ \hline
        intrusion & 1000 & 21 & 30\_1000 & 200 & 0.01 & 1 & fc\_2l \\ \hline
    \end{tabular}}
\caption{Best Hyperparamters used at different seeds for our experiments, see our repo for details.}
\end{table}

\end{document}